\documentclass[11pt]{article}

\usepackage[final]{acl}

\usepackage{times}
\usepackage{latexsym}
\usepackage{amsmath} 
\usepackage{amsthm}
\usepackage{amssymb}
\usepackage[normalem]{ulem}
\useunder{\uline}{\ul}{}
\usepackage[inline]{enumitem}
\usepackage{enumitem}
\usepackage[T1]{fontenc}

\usepackage[utf8]{inputenc}

\usepackage{microtype}

\usepackage{inconsolata}

\usepackage{graphicx}
\usepackage{amsmath}
\usepackage{xcolor}
\usepackage[utf8]{inputenc}
\usepackage{booktabs} 
\usepackage{multirow} 
\usepackage{geometry} 
\newtheorem{proposition}{Proposition}
\newtheorem{lemma}{Lemma}
\newtheorem{definition}{Definition}

\newtheorem{remark}{Remark}
\newtheorem{corollary}{Corollary}
\usepackage{multirow} 
\title{Forget What’s Sensitive, Remember What Matters: Token-Level Differential Privacy in Memory Sculpting for Continual Learning}

\author{
 \textbf{Bihao Zhan\textsuperscript{1}},
 \textbf{Jie Zhou\textsuperscript{1}}\thanks{\ \ Corresponding author, jzhou@cs.ecnu.edu.cn.},
 \textbf{Junsong Li\textsuperscript{1}},
 \textbf{Yutao Yang\textsuperscript{1}},
 \textbf{Shilian Chen\textsuperscript{1}},
 \textbf{Qianjun Pan\textsuperscript{1}},\\
 \textbf{Xin Li\textsuperscript{2}},
 \textbf{Wen Wu\textsuperscript{1}},
 \textbf{Xingjiao Wu\textsuperscript{1}},
 \textbf{Qin Chen\textsuperscript{1}},
 \textbf{Hang Yan\textsuperscript{3}},
 \textbf{Liang Dou\textsuperscript{1}},
 \textbf{Liang He\textsuperscript{1}},
\\
 \textsuperscript{1}East China Normal University, \textsuperscript{2}Shanghai AI Laboratory,  \textsuperscript{3}Shanghai Qiji Zhifeng Co., Ltd. \\
}

\begin{document}
\maketitle
\begin{abstract} 
  Continual Learning (CL) models, while adept at sequential knowledge acquisition, face significant and often overlooked privacy challenges due to accumulating diverse information. Traditional privacy methods, like a uniform Differential Privacy (DP) budget, indiscriminately protect all data, leading to substantial model utility degradation and hindering CL deployment in privacy-sensitive areas. To overcome this, we propose a privacy-enhanced continual learning (\texttt{PeCL}) framework that forgets what's sensitive and remembers what matters. Our approach first introduces a token-level dynamic Differential Privacy strategy that adaptively allocates privacy budgets based on the semantic sensitivity of individual tokens. This ensures robust protection for private entities while minimizing noise injection for non-sensitive, general knowledge. Second, we integrate a privacy-guided memory sculpting module. This module leverages the sensitivity analysis from our dynamic DP mechanism to intelligently forget sensitive information from the model's memory and parameters, while explicitly preserving the task-invariant historical knowledge crucial for mitigating catastrophic forgetting. Extensive experiments show that \texttt{PeCL} achieves a superior balance between privacy preserving and model utility, outperforming baseline models by maintaining high accuracy on previous tasks while ensuring robust privacy.
\end{abstract}

\section{Introduction}
The rapidly expanding field of Continual Learning (CL) seeks to enable models to acquire and integrate new knowledge sequentially without forgetting what they've already learned, much like human intelligence \cite{yang2025recent,shi2024continual}. This capability is essential for real-world applications where data streams are continuous and dynamic, such as in personalized recommendation systems, autonomous driving, and healthcare diagnostics. However, as CL models, particularly powerful Large Language Models (LLMs), continuously assimilate diverse and evolving datasets, they inherently accumulate vast amounts of information. Much of this information can contain sensitive personal or proprietary data \cite{carlini2021extracting,charles2024fine}. This inherent characteristic poses significant and often overlooked privacy challenges, raising concerns about potential data leakage and misuse. 

Research shows that LLMs can inadvertently memorize sensitive or personally identifiable information (PII) during training or fine-tuning \cite{meng2025rr,kuo2025proactive}. Once embedded, such data become extremely difficult to audit, modify, or erase, posing significant legal and ethical challenges \cite{liu2025rethinking,wang2024machine}. Differential Privacy (DP) offers a mathematically rigorous solution by limiting the influence of individual training samples on model outputs \cite{dwork2006differential}. While promising, applying DP to LLMs, particularly in continual learning settings, poses significant challenges. Standard methods like DPSGD \cite{abadi2016deep} and federated approaches like DP-FedEXP \cite{takakura2025accelerating} typically apply noise globally, often degrading performance in diverse tasks. While PMixED \cite{flemings2024differentially} introduces an effective token-level DP baseline for next-token prediction, it is designed for static datasets and does not address catastrophic forgetting in continual learning (CL).
These limitations are exacerbated in batch-incremental continual learning, as data arrive sequentially in this setting, which most DP methods designed for static datasets and single-round training are ill-equipped to handle \cite{abadi2016deep}.

Specifically within CL, some works have investigated private data replay or private memory management to mitigate catastrophic forgetting while preserving privacy. 
In such scenarios, privacy risks can accumulate across training phases as each task introduces different types of sensitive information \cite{desai2021continual}. Recent alternative methods exploring synthetic data \cite{murtaza2023synthetic} or computer vision-focused subnetwork isolation like PALL \cite{ozdenizci2025privacy} are not directly applicable to token-level privacy in LLMs. Furthermore, as theoretically proven by Chourasia and Shah \cite{chourasia2023forget}, achieving selective forgetting without degrading overall model utility remains fundamentally challenging.
This capability is essential for trustworthy and privacy-compliant continual learning.
However, these methods often struggle to differentiate between the varying levels of sensitivity within data, leading to either insufficient protection for highly sensitive information or excessive noise injection for general knowledge. 

The primary challenges in achieving privacy-enhanced continual learning lie in three key areas. Firstly, effectively discerning and quantifying the sensitivity of different pieces of information at a fine-grained level (e.g., token-level in text) remains a significant hurdle. Without this nuanced understanding, a blanket privacy approach either over-protects non-sensitive data, leading to unnecessary utility loss, or under-protects truly sensitive data, resulting in privacy breaches. Secondly, balancing the often conflicting goals of privacy protection and knowledge retention (i.e., mitigating catastrophic forgetting) is a complex optimization problem. Injecting noise for privacy can inadvertently corrupt crucial historical knowledge, making it difficult for the model to recall past tasks. Finally, integrating a privacy mechanism seamlessly into the CL paradigm, where knowledge acquisition is sequential and dynamic, requires a novel architectural design that can adapt its privacy enforcement based on the evolving nature of the learned information.

To address these challenges, we propose a novel privacy-enhanced continual learning (\texttt{PeCL}) framework that champions the principle of ``forgetting what's sensitive and remembering what matters". Our approach is built upon two core innovations. First, we introduce a token-level dynamic Differential Privacy strategy that adaptively allocates privacy budgets based on the semantic sensitivity of individual tokens. This mechanism intelligently identifies and quantifies the privacy risk associated with each token, ensuring robust protection for private entities while minimizing noise injection for non-sensitive, general knowledge. Second, we integrate a privacy-guided memory sculpting module. This module leverages the fine-grained sensitivity analysis from our dynamic DP mechanism to intelligently and selectively forget sensitive information from the model's memory and parameters. Crucially, it simultaneously and explicitly preserves the task-invariant historical knowledge that is vital for mitigating catastrophic forgetting. We conduct extensive experiments across various continual learning benchmarks, demonstrating that our method achieves a superior balance between privacy protection and model utility. Our approach significantly outperforms baseline models by maintaining high accuracy on previous tasks while ensuring robust privacy guarantees.

In summary, our paper makes the following three main contributions:
\begin{itemize}[leftmargin=*, align=left]
    \item We propose a novel token-level dynamic Differential Privacy strategy that adaptively allocates privacy budgets based on the semantic sensitivity of individual tokens, enabling fine-grained and efficient privacy protection in continual learning.
    \item We introduce a privacy-guided memory sculpting module that intelligently forgets sensitive information from the model's memory while explicitly preserving task-invariant historical knowledge crucial for mitigating catastrophic forgetting.
    \item We empirically demonstrate that our framework achieves a superior balance between privacy protection and model utility, significantly outperforming state-of-the-art baselines in continual learning scenarios.
\end{itemize}

\section{Related Work}
\subsection{Privacy-Preserving in LLMs}

The growing deployment of LLMs in privacy-sensitive domains has intensified interest in privacy-preserving training techniques. A key line of research applies DP to large-scale models, with early methods such as DPSGD \cite{abadi2016deep} and recent advances like large-scale DP pretraining \cite{yu2021differentially}. Moreover, DP-MLM \cite{meisenbacher2024dp} introduces per-token protection via masked prediction and differential privacy constraints, improving the fidelity of downstream tasks while maintaining privacy. Others have explored attention-based perturbations to preserve privacy \cite{huang2022attention} or enforcing local differential privacy at inference time, such as segmentation denoising \cite{mai2023split} . PMixED \cite{flemings2024differentially} provides a robust token-level DP baseline for next-token prediction, but it is limited to static datasets rather than continual learning. Meanwhile, anti-learning-based techniques such as ForgetMeNot \cite{feldman2020does} and SISA \cite{bourtoule2021machine} attempt to remove data effects retroactively, but typically require retraining or model partitioning, which is not feasible in parameter-efficient continuous learning settings.
Unlike these approaches, we integrate token-level sensitivity estimation and dynamic perturbations directly into the continuous learning phase. Furthermore, we introduce selective anti-learning capabilities via memory shaping, enabling privacy-aware adaptation without expensive retraining.


\begin{figure*}[t]
\centering
\includegraphics[width=1.0 \textwidth]{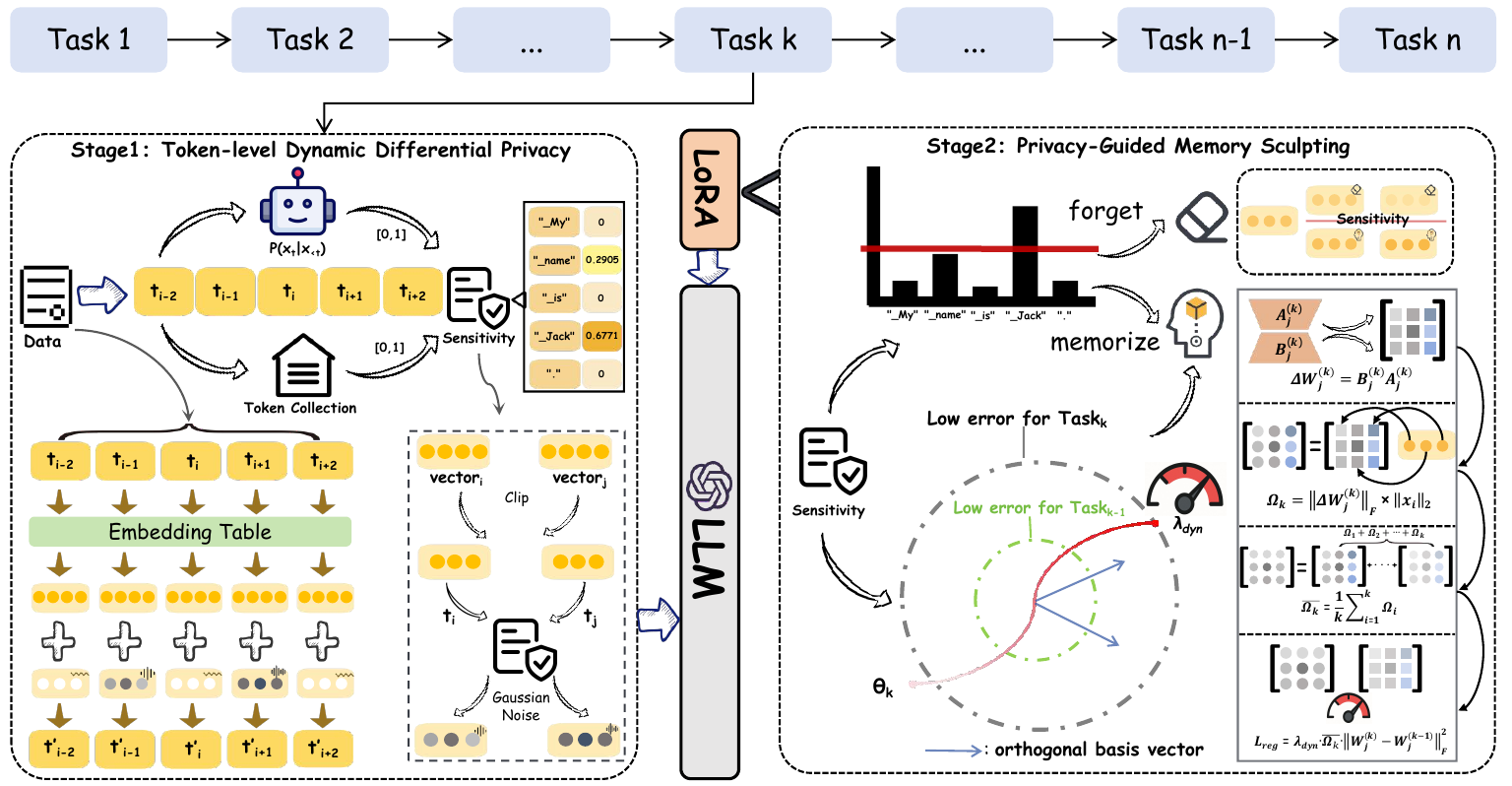} 
\caption{The framework of our \texttt{PeCL}, which is designed to balance privacy and utility by dynamically applying token-level differential privacy and intelligently sculpting model memory to retain crucial knowledge while forgetting sensitive information. The framework comprises two core modules: (a) Token-level Dynamic Differential Privacy, which adaptively injects noise into token embeddings based on their calculated sensitivity, and (b) Privacy-Guided Memory Sculpting, which leverages this sensitivity information to dynamically regularize parameters and implement privacy-aware unlearning, ensuring both catastrophic forgetting mitigation and robust privacy guarantees across sequential tasks.}
\label{fig1}
\end{figure*}

\subsection{Continual Learning of LLMs}
Continual learning (CL) for large language models (LLMs) aims to enable models to incrementally acquire new knowledge across multiple tasks while mitigating catastrophic forgetting \cite{yang2025recent}. Classical CL techniques, including Elastic Weight Consolidation (EWC) \cite{kirkpatrick2017overcoming}, Experience Replay (ER) \cite{rolnick2019experience}, and functional regularization \cite{gomez2022continually} , constrain parameter drift to preserve prior knowledge. However, these methods scale poorly with large models and often require full access to task-specific data or gradients.

Recent research has further explored more efficient and modular paradigms for continual learning. For instance, some works decouple parameter updates for different tasks by introducing orthogonal or low-rank subspaces, thereby enabling the effective learning and retention of new knowledge with almost no increase in model parameters \cite{wang2023orthogonal,yadav2023ties}. Other studies have focused on constructing dynamic, modular architectures that activate relevant model components for specific tasks at inference time through task routers or a combination of expert models \cite{jung2024pmoe,huai2025cl}. Such methods not only enhance the model's scalability and forward transfer capabilities but also offer new solutions for mitigating interference between tasks through explicit functional separation, allowing large models to adapt more flexibly to continuously changing data streams.
However, most CL methods ignore privacy or rely on unrealistic full data access. A pioneering exception is Desai et al. \cite{desai2021continual}, who introduced a sequence-level DP-CL algorithm using episodic memory. \texttt{PeCL} advances this paradigm by replacing coarse-grained sequence-level perturbations with fine-grained, token-level dynamic noise, and substituting standard episodic memory with privacy-guided memory sculpting to better mitigate catastrophic forgetting.

\section{Our Approach}
In this paper, we propose a privacy-preserving framework for continual learning (Figure~\ref{fig1}) that dynamically protects sensitive information at the token level through fine-grained sensitivity analysis, while actively sculpting the model’s memory to forget sensitive data and consolidate general knowledge. The framework comprises two key components: (1) Token-level Dynamic Differential Privacy, which adaptively allocates privacy budgets based on the sensitivity of each token, and (2) Privacy-Guided Memory Sculpting, a mechanism that integrates targeted forgetting of sensitive content with preservation of general knowledge, guided by the sensitivity signals from the first component.

\subsection{Task Formulation}
Formally, we consider a continual learning (CL) setup, where a sequence of tasks \(\mathcal{T}_{1}, \mathcal{T}_{2}, \ldots, \mathcal{T}_{N}\) arrives incrementally. Each task \(\mathcal{T}_{n} = \{(x_{i}, y_{i})\}\) consists of input–output pairs from distinct tasks. Let \(t = (t_{1}, \ldots, t_{n})\) denote a tokenized sequence. Instead of applying a uniform privacy budget across tokens or tasks, we assign a privacy-aware sensitivity score $\text{Score}(t_i) \in [0,1]$ to each token \(t_{i}\), based on both model behavior and corpus-level statistics. This score dynamically adjusts the local DP budget \(\epsilon_{i}\) and guides subsequent gradient perturbation and memory regularization.
The goal of this task is to learn new knowledge without forgetting the skills learned from previous tasks under the privacy protection setting.

\subsection{Token-level Dynamic Differential Privacy}
\label{sect:Token-level Dynamic Differential Privacy}
We propose a \textit{Token-level Dynamic Differential Privacy} (TDP) mechanism that adaptively calibrates privacy protection according to the sensitivity of individual tokens. Our approach dynamically adjusts the noise scale applied during training or inference based on a per-token sensitivity score, and we provide a theoretical guarantee that the resulting mechanism satisfies local differential privacy at the token level.

\paragraph{Token Sensitivity Calculation.}  
For each token \( t_i \) in an input sequence, we compute a privacy sensitivity score \( \text{Score}(t_i) \in [0,1] \), which reflects how likely the token is to contain private or task-specific information. This score is derived from a weighted combination of two complementary indicators—model uncertainty and contextual discriminativeness as follows:
\vspace{-2mm}
\begin{multline}
\label{eq:sensitivity_score}
\text{Score}(t_i) = \\
1 - \exp\!\Big( -\big( \alpha \cdot \text{Score}_1(t_i) + (1 - \alpha) \cdot \text{Score}_2(t_i) \big) \Big),
\end{multline}

where \( \alpha \in [0,1] \) is a tunable hyperparameter that balances the contributions of the two components. A larger \( \alpha \) places greater emphasis on model uncertainty (\( \text{Score}_1 \)), while a smaller \( \alpha \) prioritizes contextual informativeness (\( \text{Score}_2 \)).

The first component, \( \text{Score}_1(t_i) \), quantifies the model’s predictive uncertainty about token \( t_i \) given its preceding context \( t_{<i} \). Formally, it is defined as the negative log-likelihood under the model’s predictive distribution:
\begin{equation}
\label{eq:score1}
\text{Score}_1(t_i) = -\log P_\theta(t_i \mid t_{<i}),
\end{equation}
where \( P_\theta \) denotes the probability assigned by the model parameterized by \( \theta \). A higher value of \( \text{Score}_1(t_i) \) indicates that the token is either rare, highly context-dependent, or inconsistent with the model’s expectations—characteristics often associated with sensitive content.


The second component, \( \text{Score}_2(t_i) \), captures the token’s contextual discriminativeness across a set of tasks $\mathcal{T}_1, \dots, \mathcal{T}_N$. Inspired by the classic TF-IDF (Term Frequency-Inverse Document Frequency) algorithm \cite{salton1988term}, we adapt and extend traditional document-level statistics to a task-level formulation to measure token specificity. Unlike simple frequency-based measures, this adapted score emphasizes tokens that are strongly associated with specific tasks, as such tokens may inadvertently reveal private or task-identifying information. It is computed as:
\begin{equation}
\label{eq:score2}
\text{Score}_2(t_i) = \frac{1}{N} \sum_{n=1}^{N} p_n(t_i) \cdot \log \left( \frac{N}{1 + d(t_i)} \right),
\end{equation}
where \( p_n(t_i) \) denotes the normalized contextual salience of token \( t_i \) within task \( \mathcal{T}_n \), defined by
\begin{equation}
p_n(t_i) = \frac{f_n(t_i)}{f_n^{\max}}.
\end{equation}

Here, \( f_n(t_i) \) is the raw frequency of \( t_i \) across all samples in task \( \mathcal{T}_n \), and \( f_n^{\max} \) is the maximum token frequency observed in the same task, ensuring that \( p_n(t_i) \in [0,1] \).

To assess how broadly a token appears across tasks, we define the task-wise support count \( d(t_i) \) as the number of tasks in which \( t_i \) exhibits non-negligible relevance:
\begin{equation}
d(t_i) = \left| \left\{ n \in \{1, \ldots, N\} \,\middle|\, p_n(t_i) \geq \tau \right\} \right|,
\end{equation}
where \( \tau \in (0,1) \) (e.g., \( \tau = 0.2 \)) is a small threshold. Tokens with low \( d(t_i) \)—i.e., those that are salient in only a few tasks—are assigned higher discriminativeness scores, as they are more likely to carry task-specific or sensitive signals.

By integrating predictive uncertainty and cross-task discriminativeness, our sensitivity scoring mechanism enables fine-grained, context-aware privacy control. This, in turn, allows the TDP framework to dynamically modulate the magnitude of injected noise per token, ensuring stronger protection for high-sensitivity tokens while preserving utility for less sensitive ones.

\paragraph{Token-wise Dynamic Privacy Allocation.}
Based on the fused sensitivity score $\text{Score}(t_i)$, we define a token-wise dynamic privacy budget $\epsilon_i$ as:
\begin{equation}
\label{eq:epsilon_i}
\epsilon_i = \epsilon_{\text{lower}} + (\epsilon_{\text{upper}} - \epsilon_{\text{lower}}) \cdot (1 - \text{Score}(t_i))^2,
\end{equation}
where $\epsilon_{\text{lower}}$ and $\epsilon_{\text{upper}}$ are the minimum and maximum allowable privacy budgets. This formulation ensures that tokens with higher sensitivity (closer to 1) receive smaller privacy budgets (i.e., stronger protection), while less sensitive tokens (closer to 0) are granted larger $\epsilon_i$ to preserve utility. Following prior work on differential privacy \cite{abadi2016deep}, we set the token-level privacy budget range to $\epsilon_i \in [1, 10]$, reflecting a practical balance between model utility and privacy consistent with widely adopted DP regimes.

To enforce $(\epsilon_i, \delta)$-differential privacy at the token level, we inject calibrated Gaussian noise into the input embedding of each token. $\delta$ represents the probability that the privacy guarantee is violated and is typically set to a cryptographically small value. 
\textcolor{black}{Specifically, let $e_i \in \mathbb{R}^d$ denote the input embedding retrieved directly from the vocabulary lookup table before entering any layers. Injecting noise at this initial stage ensures local differential privacy at the source, preventing sensitive information leakage prior to contextual mixing in the self-attention mechanism.} Let $C$ be a predefined clipping norm bound.
The perturbed embedding $\tilde{e}_i$ is computed as:
\begin{equation}
\label{eq:perturbed_embedding}
\tilde{e}_i = \text{clip}(e_i, C) + \mathcal{N}(0, \sigma_i^2 I),
\end{equation}
where $\sigma_i = \frac{2 C \cdot \sqrt{2 \log(1.25/\delta)}}{\epsilon_i}$. Only tokens with a non-zero sensitivity score (and thus a defined $\epsilon_i$) receive noise. The $\text{clip}(\cdot, C)$ operation ensures that the $\ell_2$ norm of the embedding is bounded before noise injection, which is a necessary condition for satisfying DP guarantees. This mechanism allows our model to adaptively trade off between privacy and utility at a fine-grained level, directly guided by the semantic and contextual signals encoded in $\text{Score}(t_i)$. By tailoring noise strength per token, our approach achieves localized privacy protection, ensuring that each token embedding satisfies $(\epsilon_i, \delta)$-DP while maintaining overall task performance.

\begin{table*}[t] 
\centering
\small 
\setlength{\tabcolsep}{8pt} 
\begin{tabular}{lcccccc@{\hskip 8pt}ccc} 
\toprule 
\textbf{Method} & \textbf{Task1} & \textbf{Task2} & \textbf{Task3} & \textbf{Task4} & \textbf{Task5} & \textbf{Task6} & \textbf{BWT} & \textbf{Last} & \textbf{Avg} \\
\midrule 
SeqFT+DPSGD     & 0.141     & 0.293     & 0.633    & 0.345     & 0.149     & 0.493     & -0.116  & 0.342   & 0.403     \\
ER+DPSGD        & 0.438     & 0.463     & 0.536    & 0.443     & 0.152     & 0.573     & -0.136  & 0.434   & 0.486     \\
EWC+DPSGD       & 0.444     & 0.284     & 0.471    & 0.302     & 0.160     & 0.481     & -0.137  & 0.358    & 0.393     \\
GEM+DPSGD       & 0.337     & 0.333     & 0.394    & 0.339     & 0.179     & 0.356     & -0.157  & 0.323   & 0.387     \\
OLora+DPSGD     & 0.297     & 0.338     & 0.438     & 0.322    & 0.164     & 0.405     &-0.166   & 0.327   & 0.395     \\
SeqFT+FN        & 0.101     & 0.204     & 0.145    & 0.200       & 0.135     & 0.331     & -0.306   & 0.186    & 0.351     \\
ER+FN           & \textbf{0.456}      & 0.420      & 0.673    & 0.371     & 0.029     & 0.309    & -0.121  & 0.376  & 0.492     \\
EWC+FN          & 0.260      & 0.231     & 0.208    & 0.236     & 0.261     & 0.183     & -0.248  & 0.230   & 0.397     \\
GEM+FN          & 0.365     & 0.253     & 0.519    & 0.276     & 0.285     & 0.210     & -0.193   & 0.318  & 0.428     \\
OLora+FN        & 0.329     & 0.192     & 0.240    & 0.177     & 0.106     & 0.165     &-0.294   & 0.202   & 0.363     \\ 
\textcolor{black}{PMixED}       & 0.431     & 0.110     & 0.737    & 0.338     & 0.376     & 0.419     &-0.118   & 0.402   & 0.463    \\ \midrule
\texttt{\textbf{PeCL}} (Ours)             & 0.436     & \textbf{0.521}     & \textbf{0.769}   & \textbf{0.444}     & \textbf{0.456}     & \textbf{0.714}     & \textbf{-0.093}  & \textbf{0.573}  & \textbf{0.535}    \\ \midrule
MTL+DPSGD (Upper Bound)       & 0.456     & 0.538     & 0.780    & 0.482     & 0.197     & 0.334     & -     & 0.464    & - \\ 
MTL+FN (Upper Bound)         & 0.405     & 0.463     & 0.808    & 0.448     & 0.503     & 0.305     & -     & 0.489     & -\\
\bottomrule 
\end{tabular}
\caption{Performance comparison of different methods.} 
\label{Main_Result} 
\end{table*}

\subsection{Privacy-Guided Memory Sculpting}
While token-level dynamic differential privacy provides localized privacy at the input embedding level, it does not inherently prevent sensitive information from being memorized within the model parameters over time. To address this, we introduce Privacy-guided Memory Sculpting (PMS), which reshapes parameter updates by integrating token-level privacy sensitivity into the learning dynamics. Our method comprises two complementary components: Memory Regularization and Privacy-Aware Unlearning.

\paragraph{Memory Regularization.} For each incoming task $\mathcal{T}_k$, we first compute the parameter increment for LoRA \cite{hu2022lora} adapter $j$ as:
\begin{equation}
\label{eq:delta_W}
\Delta W_j^{(k)} = B_j^{(k)} A_j^{(k)}.
\end{equation}
where $\Delta W_j^{(k)}$ means the learned update specific to task $\mathcal{T}_k$. Next, inspired by \citet{sun2023simple},we compute a task-specific importance score, $\Omega_k$, by multiplying the Frobenius norm of $\Delta W_j^{(k)}$ with the $L_2$ norm of the input activations $x$ corresponding to task $\mathcal{T}_k$:
\begin{equation}
\label{eq:omega_k}
\Omega_k = \left\| \Delta W_j^{(k)} \right\|_F \times \left\|x\right\|_2.
\end{equation}

We then accumulate the importance across tasks using an online average, which serves as a measure of the cumulative importance of previously learned knowledge:
\begin{equation}
\label{eq:omega_bar_k}
\bar{\Omega}_k = \frac{1}{k} \sum_{i=1}^{k} \Omega_i.
\end{equation}

Finally, inspired by the classic Elastic Weight Consolidation (EWC) \cite{kirkpatrick2017overcoming} approach, we define the stability-aligned regularization loss. While EWC typically relies on the Fisher Information Matrix, we adapt this regularization principle directly to the LoRA parameter space to stabilize parameters crucial for retaining historical knowledge:
\begin{equation}
\label{eq:L_reg}
\mathcal{L}_{\text{reg}} = \lambda_{\text{dyn}} \cdot \bar{\Omega}_k \cdot \left\| W_j^{(k)}-W_j^{(k-1)} \right\|_F^2,
\end{equation}
where $W_j^{(k)}$ denotes the current LoRA parameters for task $k$, and $\lambda_{\text{dyn}}$ is a dynamic regularization weight.

To introduce privacy-aware adaptivity into this regularization, we modulate the regularization strength dynamically according to the average token sensitivity score $\bar{s}_k$ of task $\mathcal{T}_k$. Specifically, we define the dynamic coefficient $\lambda_{\text{dyn}}$ as:
\begin{equation}
\label{eq:lambda_dyn}
\lambda_{\text{dyn}} = \lambda_{\text{max}} \cdot (1 - \bar{s}_k) + \lambda_{\text{min}} \cdot \bar{s}_k,
\end{equation}
where $\lambda_{\text{max}}$ and $\lambda_{\text{min}}$ represent the maximum and minimum levels of regularization strength, respectively. This formulation ensures that tasks involving less sensitive content (i.e., lower $\bar{s}_k$) are regularized more strictly to retain prior knowledge, while more privacy-sensitive tasks are allowed greater flexibility for adaptation and forgetting. The average sensitivity $\bar{s}_k$ for task $\mathcal{T}_k$ is computed as the mean of $\text{Score}(t_i)$ over all tokens $t_i$ in the task’s data.


\paragraph{Privacy-Aware Unlearning.} To further reshape learning dynamics by directly addressing sensitive information, we define a second loss term called Privacy-Aware Unlearning. This term directly adjusts token-level gradient contributions using their fused sensitivity scores. The formulation is:
\vspace{-2mm}
\begin{multline}
\label{eq:L_unlearn}
\mathcal{L}_{\text{unlearn}} = \\
\frac{1}{M} \sum_{i=1}^{M} (\text{Score}(t_i) - \theta) \cdot \ell(t_i) \cdot \mathbb{I}(\text{Score}(t_i) > \theta),
\end{multline}
where $\ell(t_i)$ is the cross-entropy loss of token $t_i$, $\theta$ is a predefined sensitivity threshold, and $M$ is the number of tokens in the sequence. $\mathbb{I}(\cdot)$ is the indicator function, meaning this term contributes only for tokens whose sensitivity $\text{Score}(t_i)$ exceeds the threshold $\theta$. This formulation specifically targets high-sensitivity tokens: by multiplying their loss contribution by $(\text{Score}(t_i) - \theta)$, we encourage the model to unlearn or softly suppress the reinforcement of information associated with these tokens.

The overall training objective combines task-specific prediction, memory regularization, and unlearning:
\begin{equation}
\label{eq:L_total}
\mathcal{L}_{\text{total}} = \mathcal{L}_{\text{task}} + \mathcal{L}_{\text{reg}} + \lambda_{\text{unlearn}} \cdot \mathcal{L}_{\text{unlearn}},
\end{equation}
where $\mathcal{L}_{\text{task}}$ is the standard task-specific loss (e.g., cross-entropy loss for classification or language modeling), and $\lambda_{\text{unlearn}}$ is a hyperparameter controlling the strength of the privacy-aware unlearning component.

This unified objective enables continual learning that dynamically aligns parameter retention and forgetting with token-level privacy sensitivity, ensuring both task stability and robust privacy guarantees.

\textcolor{black}{Together, TDP and PMS form a \textbf{Dual-Layer Defense}. While PMS triggers parameter-level unlearning for tokens above the threshold $\theta$, TDP protects ``near-threshold'' tokens. Since TDP allocates noise continuously, tokens below $\theta$ with non-zero sensitivity still receive proportional perturbation. This ``soft suppression'' ensures robust privacy at the input level even without explicit memory sculpting.}


\section{Experiments}
\subsection{Experiments Setups}
\paragraph{Datasets.} To evaluate our privacy-enhanced continual learning framework, we construct a multi-task dataset covering six distinct domains with varying privacy sensitivities. In the absence of any public benchmark for privacy-preserving continual learning, we select six tasks: FOMC \cite{shah2023trillion}, Yelp \cite{asghar2016yelp}, AGNews \cite{zhang2015character}, Amazon\footnote{\tiny https://www.kaggle.com/datasets/kritanjalijain/amazon-reviews}, Mentill\footnote{\tiny https://huggingface.co/datasets/mavinsao/reddit-mental-illnes}, and Yahoo\footnote{\tiny https://www.kaggle.com/datasets/bhavikardeshna/yahoo-email-classification}, and sample 3,000 examples for each task. \textcolor{black}{All tasks are formulated as text classification. These datasets are selected because they originate from real-world user reviews and Q\&A platforms, naturally containing unstructured and potentially sensitive personal information, making them ideal benchmarks for evaluating privacy-preserving continual learning.}

\paragraph{Evaluation Metrics.} \textbf{Evaluation Metrics.} We employ three widely used continual learning metrics: Average Accuracy (Avg), Last Accuracy (Last), and Backward Transfer (BWT) \cite{chaudhry2018riemannian,lopez2017gradient}. Avg measures the mean accuracy across all tasks at each training step, Last reflects the final average accuracy after sequentially learning all tasks, and BWT quantifies the effect of learning new tasks on previously acquired knowledge (i.e., catastrophic forgetting). Formal mathematical definitions are detailed in Appendix~\ref{appendix:evaluation_metrics}.


\paragraph{Baselines.}  To evaluate the effectiveness of our approach, we compare against several representative baselines. For continual learning, we consider ER \cite{rolnick2019experience}, EWC \cite{kirkpatrick2017overcoming}, GEM \cite{lopez2017gradient}, and O-LoRA \cite{wang2023orthogonal}, covering replay-based, regularization-based, and adapter-based strategies. For privacy protection, we include DPSGD \cite{abadi2016deep} and a frozen embedding with additive noise (FN) as baselines. Additionally, we report two control setups: Multitask training (MTL) is always regarded as the upper bound, which simultaneously learns all tasks (MTL), and Sequential Finetuning (SeqFT), which sequentially learns each task. \textcolor{black}{Furthermore, to provide a comprehensive comparison in fine-grained privacy, we also include PMixED \cite{flemings2024differentially}, a recent  token-level privacy baseline.}

\subsection{Main Results}
Table~\ref{Main_Result} presents the evaluation of various CL methods under differential privacy constraints across six tasks, measured by Backward Transfer (BWT), Last accuracy, and Average accuracy (Avg). The results yield three key insights:

First, \texttt{PeCL} achieves the best overall performance across all three metrics, demonstrating the efficacy of combining token-level dynamic privacy with memory sculpting. Notably, it surpasses the recent token-level DP baseline, PMixED, by achieving a superior balance between privacy protection and knowledge retention. As expected, SeqFT variants exhibit the most severe catastrophic forgetting, underscoring the necessity of dedicated CL mechanisms under privacy constraints.

Second, \texttt{PeCL} consistently outperforms  baselines like ER+DPSGD and ER+FN, highlighting the clear advantage of our adaptive noise allocation over static noise injection strategies. Furthermore, traditional regularization-based methods (e.g., EWC and GEM) and adapter-based approaches (like OLora) show limited effectiveness. This indicates that fixed parameter constraints or standard adapters struggle to stabilize learning under the disruptive effects of rigorous privacy noise.

Third, \texttt{PeCL} surprisingly surpasses the multitask learning baseline (MTL+DPSGD) in both Last and Avg accuracy. Achieving this in a realistic, task-incremental setting rather than a centralized joint-training scenario strongly underscores the robustness and practical applicability of our framework in privacy-sensitive continual learning environments.

\begin{table*}[t] 
\vspace{-2mm}
\centering
\small 
\setlength{\tabcolsep}{8pt} 
\begin{tabular}{lcccccc@{\hskip 8pt}ccc} 
\toprule 
Method & Task1 & Task2 & Task3 & Task4 & Task5 & Task6 & BWT & Last & Avg \\
\midrule 
\texttt{\textbf{PeCL}} (Ours)    & 0.436     & \textbf{0.521}     & \textbf{0.869}     & \textbf{0.444}     & \textbf{0.456}     & \textbf{0.714}    & \textbf{-0.093}  &\textbf{0.573}  & \textbf{0.535}     \\ \midrule
-w/o TDP   & \textbf{0.445}      & 0.313     & 0.422    & 0.297   &0.149  & 0.146    & -0.155  & 0.295   & 0.473    \\
-w/o PMS   & 0.294     & 0.224     & 0.349     & 0.219     & 0.262     & 0.710     & -0.430  & 0.312   & 0.465     \\
~~~-w/o MemReg        & 0.385     & 0.387     & 0.730      & 0.371     & 0.434     & 0.725    & -0.212  & 0.505  & 0.496     \\
~~~-w/o Unlearning            & 0.407     & 0.499     & 0.857     & 0.421     & 0.350      & 0.727    & -0.133  & 0.544  & 0.524     \\
\bottomrule 
\end{tabular}
\caption{The results of ablation studies.} 
\label{ablation_studies} 
\end{table*}


\subsection{Ablation Studies}
To assess the effectiveness of each component in our framework, we conduct an ablation study as summarized in Table~\ref{ablation_studies}. We focus on two major modules: Token-level Differential Privacy (TDP) and Privacy-Guided Memory Sculpting (PMS), as well as its two internal mechanisms: Memory Regularization (MemReg) and Privacy-Aware Unlearning. TDP serves as the foundation of our privacy design: when replaced with coarser sentence-level privacy (w/o TDP), both performance and forgetting degrade significantly, indicating that fine-grained sensitivity modulation is essential for balancing utility and protection. PMS further enhances long-term performance, as removing it entirely (w/o PMS) leads to the largest drop in average accuracy and a sharp increase in forgetting. Within PMS, Memory Regularization and Unlearning target complementary aspects: disabling MemReg results in weakened task stability, while removing the unlearning component slightly reduces accuracy but notably increases forgetting. In summary, TDP enables precise privacy modulation, MemReg stabilizes knowledge across tasks, and Unlearning effectively removes sensitive traces, together forming a cohesive and robust privacy-preserving continual learning framework.

\begin{figure}
\centering
\includegraphics[width=1 \columnwidth]{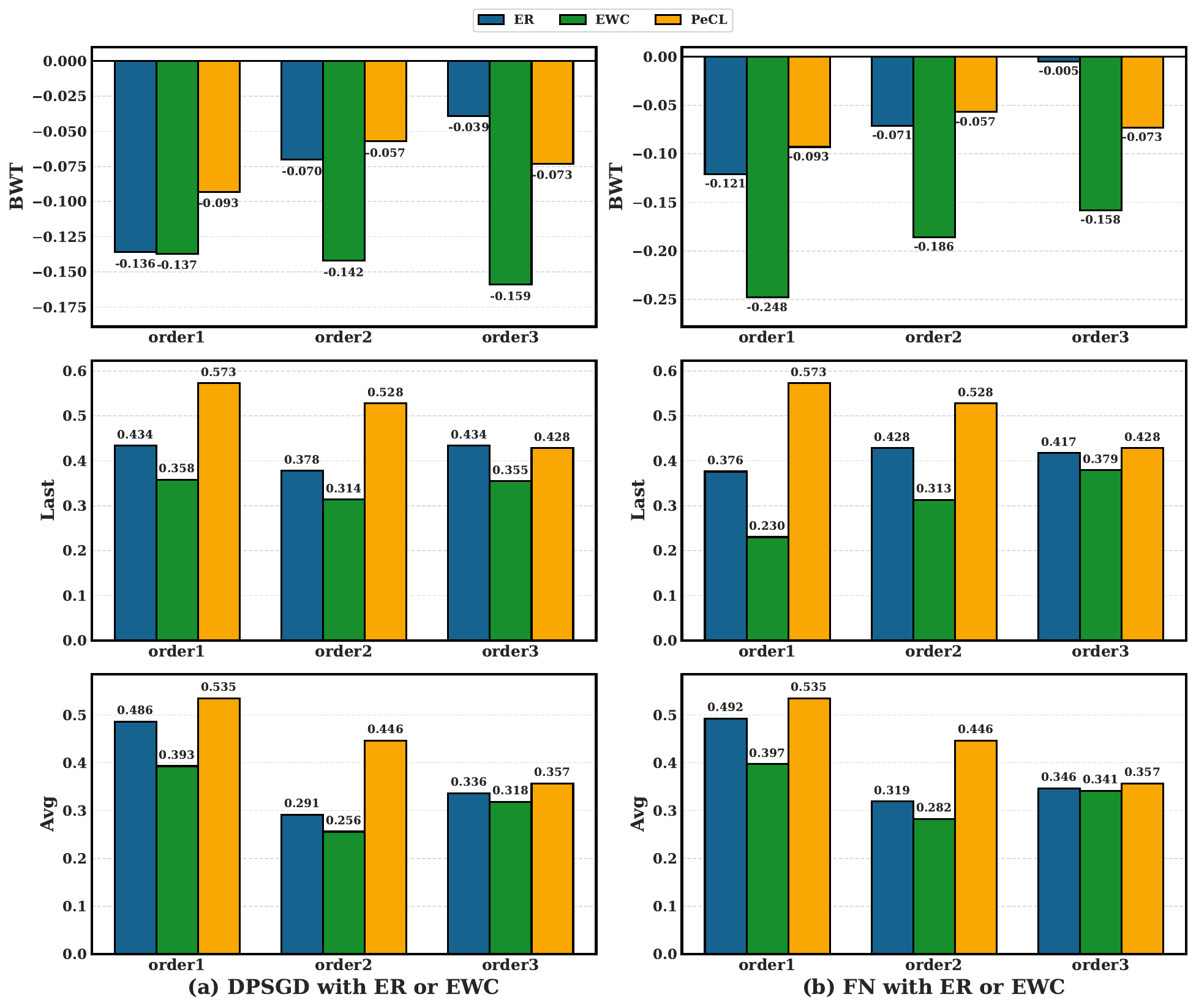} 
\caption{Results of different task order.}
\label{fig2}
\end{figure}


\subsection{Further Analysis}
In this section, we conduct extensive sensitivity analyses to gain deeper insights into the proposed method. We primarily focus on the Influence of Task Order to evaluate the model's robustness against varying task arrival sequences, as well as the Impact of Hyperparameter $\alpha$ to understand the balance between uncertainty and context in privacy sensitivity scoring. For additional analyses, including the Impact of Hyperparameter $\lambda_{\mathrm{unlearn}}$ and Impact of Hyperparameter $\alpha$ and $\theta$, please refer to the Appendix~\ref{Anlysis of Hyperparameter}.
\paragraph{Influence of Task Order.} We consider three task orders: a natural progression (order1: 1 $\rightarrow$ 2 $\rightarrow$ ... $\rightarrow$ 6), the reversed sequence (order2: 6 $\rightarrow$ 5 $\rightarrow$... $\rightarrow$ 1), and a deliberately shuffled permutation (order3: 4 $\rightarrow$ 5 $\rightarrow$ 1 $\rightarrow$ 3 $\rightarrow$ 6 $\rightarrow$ 2). Figure~\ref{fig2} illustrates how different task arrival sequences affect model performance. Across all task orders, our \texttt{PeCL} consistently outperforms baselines such as ER+DPSGD and EWC+DPSGD in terms of both accuracy and forgetting. Notably, \texttt{PeCL} exhibits strong stability in order1 and order2, maintaining high final accuracy and low forgetting. While there is a slight performance drop in order3, \texttt{PeCL} still remains competitive and significantly more robust than other methods, whose results vary more dramatically across different orders. These results demonstrate that \texttt{PeCL} can adapt well to varying task sequences.


\section{Conclusions and Future Work}
In this paper, we present \texttt{PeCL}, a novel privacy-enhanced continual learning framework designed to address the critical challenges of data privacy and catastrophic forgetting in sequential learning scenarios, particularly with LLMs. Our approach tackles this challenge through two key innovations: a token-level dynamic Differential Privacy strategy and a privacy-guided memory sculpting module.
Extensive experiments demonstrate that \texttt{PeCL} achieves superior performance compared to state-of-the-art baselines, delivering higher average accuracy, better retention of past knowledge, and stronger privacy guarantees. Ablation studies confirm the necessity of the main components, while further analysis shows the robustness of our method across varying task orders and hyperparameter settings. 
The results highlight the importance of integrating sensitivity-aware, adaptive privacy mechanisms into continual learning systems.

\section*{Limitations}


\textcolor{black}{While PeCL establishes a strong foundation, several challenges remain. First, the privacy budget $\epsilon$ accumulates in long-context scenarios, potentially exceeding strict bounds. Future work will explore Rényi Differential Privacy (RDP) to establish tighter boundaries for extended sequences. Second, beyond theoretical DP guarantees, empirical assessments against membership inference (MIA) and data extraction attacks are needed to validate practical robustness. Finally, transitioning our heuristic-based sensitivity scoring to formal principles, such as Fisher Information or the Information Bottleneck, would further enhance theoretical rigor.}

\section*{Ethical considerations}
The six datasets used in the experiments including FOMC, Yelp, AGNews,  Amazon, Mentill and Yahoo are widely used datasets. 
Our research strictly adheres to the Code of Ethics, particularly regarding data privacy, transparency, and responsible computing practices. And there is no participant involved.

\bibliography{custom}

\newpage
\appendix

\section{Evaluation Metrics}
\label{appendix:evaluation_metrics}

To rigorously evaluate the performance of our privacy-enhanced continual learning framework, we employ three standard metrics: Backward Transfer (BWT), Last Accuracy (Last), and Average Accuracy (Avg). Let $N$ denote the total number of tasks, and $R_{k,i}$ denote the test accuracy on task $i$ after the model has been sequentially trained on tasks $1$ through $k$.

\textbf{Backward Transfer (BWT).} BWT quantifies the effect of learning new tasks on previously acquired knowledge. A negative BWT indicates catastrophic forgetting, while a positive BWT indicates forward knowledge transfer. It is defined as:
\begin{equation}
    \mathrm{BWT} = \frac{1}{N-1}\sum_{i=1}^{N-1}\bigl(R_{N,i} - R_{i,i}\bigr)
\end{equation}
where $R_{N,i}$ is the accuracy on task $i$ after training on all $N$ tasks, and $R_{i,i}$ is the accuracy on task $i$ immediately after it was first learned.

\textbf{Last Accuracy (Last).} Last accuracy measures the model's final average performance across all $N$ tasks after the entire sequential training process is complete:
\begin{equation}
    \mathrm{Last} = \frac{1}{N}\sum_{i=1}^{N}R_{N,i}
\end{equation}

\textbf{Average Accuracy (Avg).} Avg measures the mean accuracy over all evaluated tasks at every training step, providing a comprehensive view of the model's performance trajectory throughout the continual learning process:
\begin{equation}
    \mathrm{Avg} = \frac{1}{N}\sum_{k=1}^{N}\left(\frac{1}{k}\sum_{i=1}^{k}R_{k,i}\right)
\end{equation}

\section{Analysis of Hyperparameter}
\label{Anlysis of Hyperparameter}
\paragraph{Impact of Hyperparameter $\alpha$.} 
We investigate the effect of the balancing coefficient $\alpha \in [0,1]$, which controls the trade-off between model uncertainty $\text{Score}_1(t_i)$ and contextual informativeness $\text{Score}_2(t_i)$ in the privacy sensitivity score. As shown in Figure~\ref{fig3} and Figure~\ref{fig5}a, both the Avg performance and BWT reach their optimal values when $\alpha = 0.5$, suggesting that a balanced combination of the two factors leads to the best performance. Increasing $\alpha$ beyond $0.5$ (favoring uncertainty) results in degraded BWT, indicating more forgetting. Conversely, lower $\alpha$ values (favoring context) also hurt Avg, likely due to insufficient attention to uncertain tokens. These results validate the effectiveness of integrating both uncertainty and contextual features, with $\alpha = 0.5$ serving as a robust choice.

\paragraph{Impact of Hyperparameter $\lambda_{\mathrm{unlearn}}$.} 
We study the influence of the hyperparameter $\lambda_{\mathrm{unlearn}}$, which controls the strength of the unlearning regularization term. As shown in Figure~\ref{fig4}, setting $\lambda_{\mathrm{unlearn}} = 1$ achieves the best trade-off between stability and performance, yielding the highest average score ($0.535$) and the lowest forgetting as measured by BWT ($-0.093$). When $\lambda_{\mathrm{unlearn}} = 0$, i.e., no unlearning applied, BWT degrades to $-0.133$, showing more severe forgetting. As $\lambda$ increases beyond 1, the Avg score drops (e.g., $0.454$ at $\lambda=5$), likely due to over-regularization that hurts forward transfer. These results suggest that a moderate unlearning strength is beneficial for mitigating forgetting without compromising overall task performance.

\paragraph{Impact of Hyperparameter $\theta$.}
We investigate the impact of the sensitivity threshold $\theta$ in the Privacy-Aware Unlearning loss, which determines which tokens are softly suppressed during training. As shown in Figure~\ref{fig5}b, $\theta = 0.6$ achieves the best trade-off, yielding the highest Avg ($0.535$) and lowest forgetting (BWT $= -0.093$). A lower threshold (e.g., $\theta = 0.3$) suppresses too many tokens, hurting performance (Avg $= 0.514$, BWT $= -0.108$), while a higher one (e.g., $\theta = 0.9$) retains overly sensitive tokens (Avg $= 0.523$, BWT $= -0.105$). These results suggest that a moderate $\theta$ best balances utility and privacy at the token level.

\section{Per-Task Performance of Task Order} 
Table~\ref{Table 1} lists the full classification accuracy of each task for the continual learning methods that performed better in different orders. The six columns in the middle correspond to tasks 1-6, and the three columns on the far right provide BWT, Last, and Avg.

\section{Per-Task Performance of Hyperparameter $\alpha$.} 
Table~\ref{Table 2} lists the complete classification accuracy of our method on each task under different hyperparameters $\alpha$. The middle six columns correspond to tasks 1-6, and the rightmost three columns provide BWT, Last, and Avg.

\section{Per-Task Performance of Hyperparameter $\lambda_\text{unlearn}$.} 
Table~\ref{Table 3} lists the complete classification accuracy of our method on each task under different hyperparameters $\lambda_\text{unlearn}$. The middle six columns correspond to tasks 1-6, and the rightmost three columns provide BWT, Last, and Avg.

\begin{figure}
\centering
\includegraphics[width=1 \columnwidth]{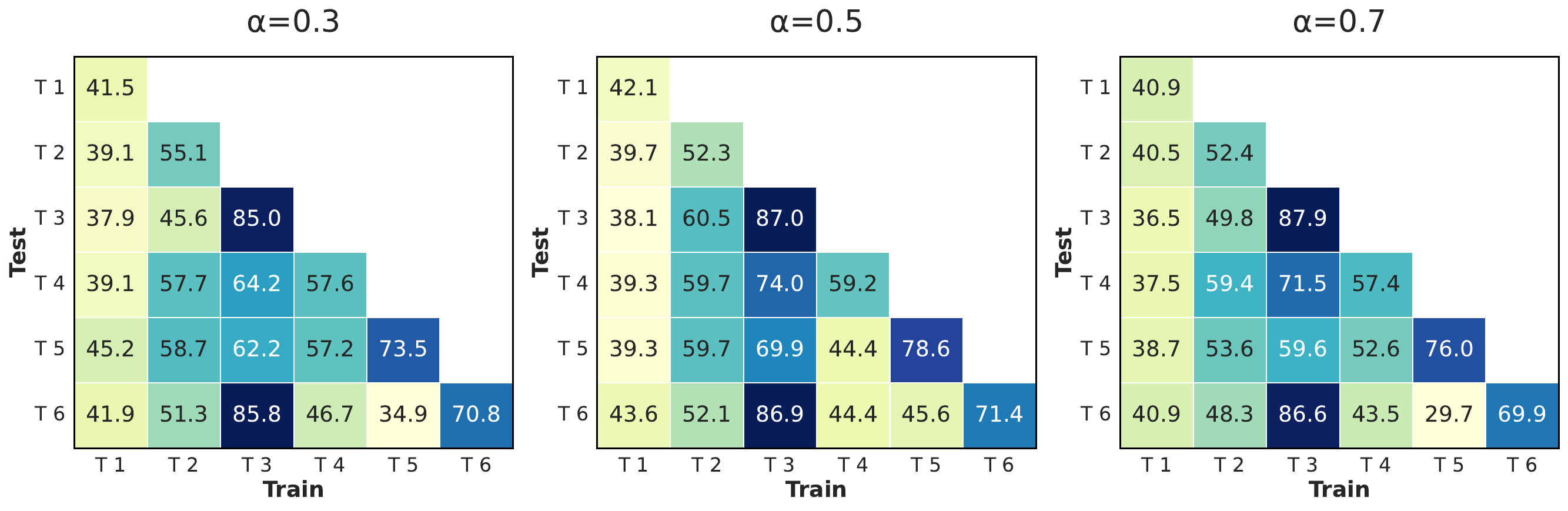} 
\caption{The performance of $\alpha$.}
\label{fig3}
\end{figure}

\begin{table*}[t] 
\small
\centering
\setlength{\tabcolsep}{8pt} 
\small
\begin{tabular}{llcccccc@{\hskip 8pt}ccc} 
\toprule 
\textbf{Task Order} & \textbf{Method} & \textbf{task1} & \textbf{task2} & \textbf{task3} & \textbf{task4} & \textbf{task5} & \textbf{task6} & \textbf{BWT} & \textbf{Last} & \textbf{Avg} \\
\midrule 
\multirow{5}{*}{order1} 
 & ER+DPSGD  & 0.438 & 0.463 & 0.536 & 0.443 & 0.152 & 0.573 & -0.136 & 0.434 & 0.486 \\
 & EWC+DPSGD & 0.444 & 0.284 & 0.471 & 0.302 & 0.160 & 0.481 & -0.137 & 0.358 & 0.393 \\
 & ER+FN     & 0.456 & 0.420 & 0.673 & 0.371 & 0.029 & 0.309 & -0.121 & 0.376 & 0.492 \\
 & EWC+FN    & 0.260 & 0.231 & 0.208 & 0.236 & 0.261 & 0.183 & -0.248 & 0.230 & 0.397 \\
 & \texttt{PeCL}      & 0.436 & 0.521 & 0.869 & 0.444 & 0.456 & 0.714 & -0.093 & 0.573 & 0.535 \\
\midrule 
\multirow{5}{*}{order2} 
 & ER+DPSGD  & 0.155 & 0.139 & 0.450 & 0.575 & 0.473 & 0.476 & -0.070 & 0.378 & 0.291 \\
 & EWC+DPSGD & 0.142 & 0.117 & 0.414 & 0.312 & 0.419 & 0.480 & -0.142 & 0.314 & 0.256 \\
 & ER+FN     & 0.148 & 0.364 & 0.470 & 0.591 & 0.468 & 0.526 & -0.071 & 0.428 & 0.319 \\
 & EWC+FN    & 0.104 & 0.123 & 0.347 & 0.502 & 0.322 & 0.480 & -0.186 & 0.313 & 0.282 \\
 & \texttt{PeCL}      & 0.276 & 0.368 & 0.541 & 0.798 & 0.618 & 0.567 & -0.057 & 0.528 & 0.446 \\
\midrule
\multirow{5}{*}{order3} 
 & ER+DPSGD  & 0.174 & 0.493 & 0.450 & 0.657 & 0.283 & 0.546 & -0.039 & 0.434 & 0.336 \\
 & EWC+DPSGD & 0.089 & 0.462 & 0.262 & 0.496 & 0.302 & 0.516 & -0.159 & 0.355 & 0.318 \\
 & ER+FN     & 0.184 & 0.493 & 0.558 & 0.601 & 0.147 & 0.521 & -0.005 & 0.417 & 0.346 \\
 & EWC+FN    & 0.126 & 0.528 & 0.427 & 0.426 & 0.179 & 0.590 & -0.158 & 0.379 & 0.341 \\
 & \texttt{PeCL}      & 0.233 & 0.555 & 0.300 & 0.522 & 0.359 & 0.598 & -0.073 & 0.428 & 0.357 \\
\bottomrule 
\end{tabular}
\caption{Per-Task performance of task order.} 
\label{Table 1} 
\end{table*}

\begin{table*}[t]
\small
\centering
\begin{tabular}{lccccccccc}
\toprule
\textbf{} & \textbf{task1} & \textbf{task2} & \textbf{task3} & \textbf{task4} & \textbf{task5} & \textbf{task6} & \textbf{BWT} & \textbf{Last} & \textbf{Avg} \\
\midrule
$\alpha$=0.2 & 0.395 & 0.485 & 0.865 & 0.423 & 0.330 & 0.705 & -0.130 & 0.534 & 0.518 \\
$\alpha$=0.3 & 0.419 & 0.513 & 0.858 & 0.467 & 0.349 & 0.708 & -0.104 & 0.552 & 0.523 \\
$\alpha$=0.5 & 0.436 & 0.521 & 0.869 & 0.444 & 0.456 & 0.714 & -0.093 & 0.573 & 0.535 \\
$\alpha$=0.7 & 0.409 & 0.483 & 0.866 & 0.435 & 0.297 & 0.699 & -0.121 & 0.532 & 0.510 \\
$\alpha$=0.8 & 0.401 & 0.427 & 0.874 & 0.378 & 0.253 & 0.699 & -0.148 & 0.505 & 0.489 \\
\bottomrule
\end{tabular}
\caption{Per-Task Performance of Hyperparameter $\alpha$.}
\label{Table 2}
\end{table*}

\begin{table*}[t]
\small
\centering
\begin{tabular}{lccccccccc}
\toprule
\textbf{$\lambda_\text{unlearn}$} & \textbf{task1} & \textbf{task2} & \textbf{task3} & \textbf{task4} & \textbf{task5} & \textbf{task6} & \textbf{BWT} & \textbf{Last} & \textbf{Avg} \\
\midrule
0  & 0.407 & 0.499 & 0.857 & 0.421 & 0.350 & 0.727 & -0.133 & 0.544 & 0.524 \\
1  & 0.436 & 0.521 & 0.869 & 0.444 & 0.456 & 0.714 & -0.093 & 0.573 & 0.535 \\
3  & 0.395 & 0.489 & 0.864 & 0.431 & 0.305 & 0.718 & -0.127 & 0.534 & 0.511 \\
5  & 0.315 & 0.429 & 0.846 & 0.396 & 0.314 & 0.706 & -0.101 & 0.501 & 0.454 \\
7  & 0.393 & 0.482 & 0.850 & 0.439 & 0.413 & 0.718 & -0.093 & 0.549 & 0.513 \\
10 & 0.440 & 0.471 & 0.850 & 0.423 & 0.333 & 0.715 & -0.059 & 0.539 & 0.476 \\
\bottomrule
\end{tabular}
\caption{Per-Task Performance of Hyperparameter $\lambda_\text{unlearn}$.}
\label{Table 3}
\end{table*}

\begin{figure}
\centering
\includegraphics[width=1 \columnwidth]{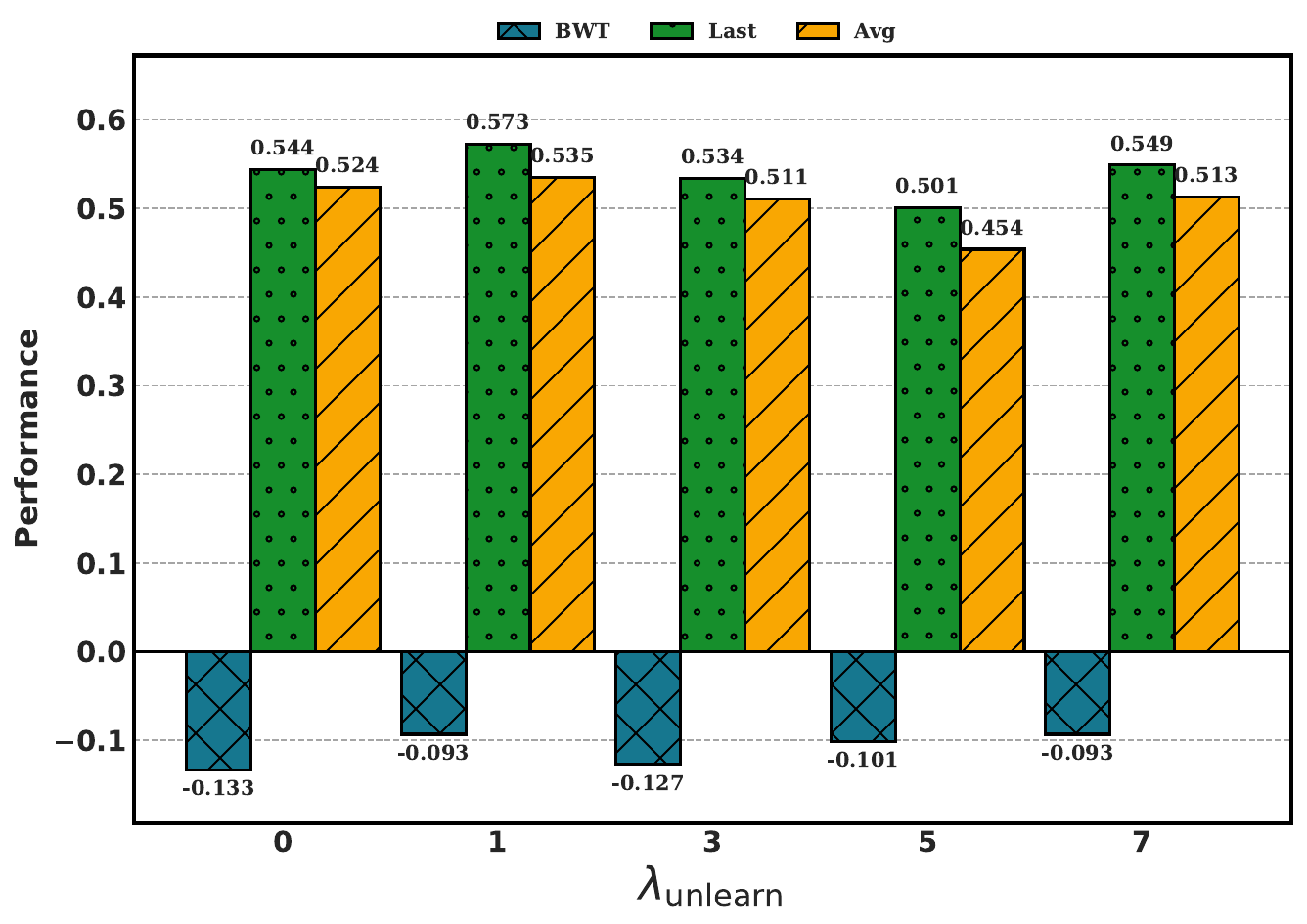} 
\caption{The Impact of $\lambda_{\mathrm{unlearn}}$.}
\label{fig4}
\vspace{-4mm}
\end{figure}

\begin{figure}
\centering
\includegraphics[width=1 \columnwidth]{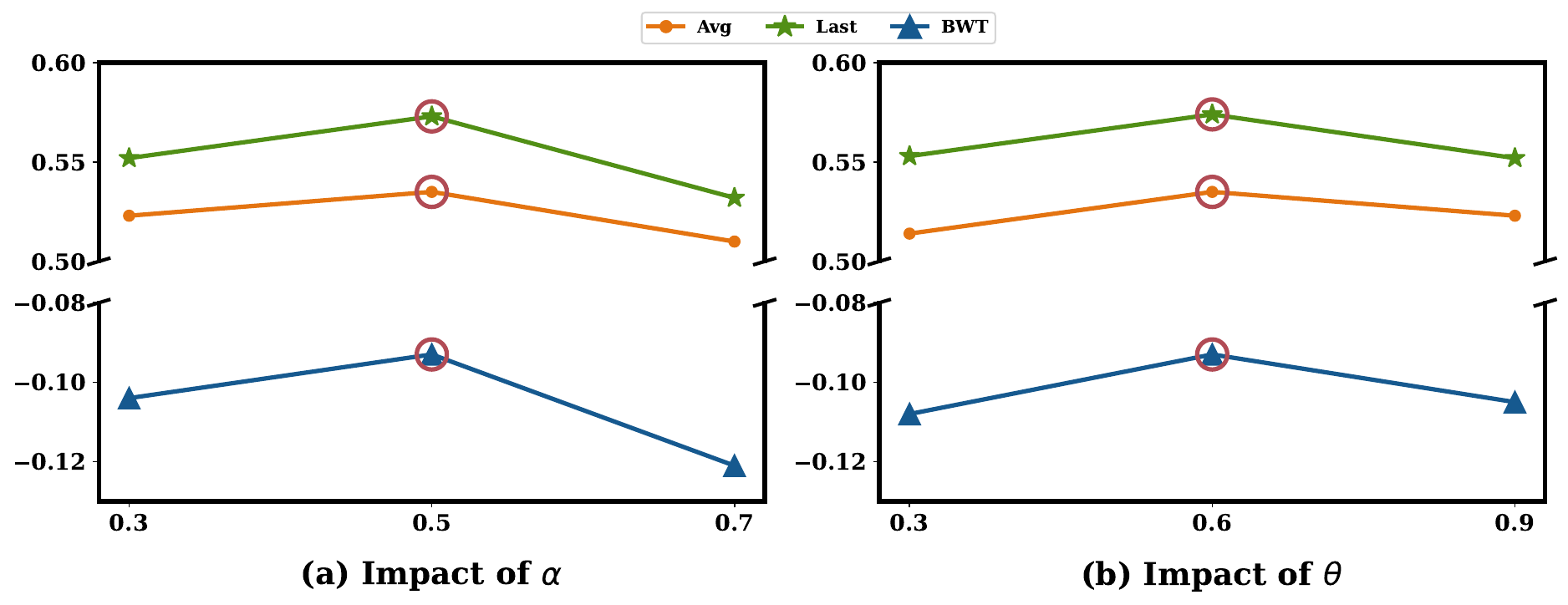} 
\caption{The Impact of Hyperparameter $\alpha$ and $\theta$.}
\label{fig5}
\vspace{-4mm}
\end{figure}

\section{Baseline Details.}
To evaluate the effectiveness of our approach, we compare it against a diverse set of representative baselines spanning different methodological paradigms.

For continual learning, we consider four established methods:
\begin{itemize}[leftmargin=*, align=left]
\item \textbf{Experience Replay (ER)}~\cite{rolnick2019experience}, a replay-based approach that stores past examples in a buffer for rehearsal. When learning a new task, ER mixes old and new data for joint training, allowing the model to revisit learned knowledge to combat forgetting.
\item \textbf{Elastic Weight Consolidation (EWC)}~\cite{kirkpatrick2017overcoming}, a regularization-based method that penalizes changes to parameters important for previous tasks. EWC identifies these key parameters using the Fisher Information Matrix and adds a penalty term to the loss function to constrain their drift while learning new tasks.
\item \textbf{Gradient Episodic Memory (GEM)}~\cite{lopez2017gradient}, which constrains gradient updates to avoid increasing the loss on previously seen tasks. Specifically, if a gradient update for the current task conflicts with the gradient directions from stored past examples, GEM projects it to a non-conflicting direction.
\item \textbf{O-LoRA}~\cite{wang2023orthogonal}, an adapter-based strategy that leverages orthogonal low-rank adaptation modules to mitigate interference across tasks. It assigns an independent LoRA module to each task and enforces an orthogonality constraint, ensuring that parameter updates for different tasks occur in their own subspaces to achieve knowledge isolation.
\end{itemize}

For privacy-preserving learning, we include two baselines:
\begin{itemize}[leftmargin=*, align=left]
\item \textbf{Differentially Private Stochastic Gradient Descent (DPSGD)}~\cite{abadi2016deep}, which injects calibrated noise into gradients to enforce differential privacy. Before each update, this method first clips the L2 norm of per-sample gradients to bound their sensitivity and then adds Gaussian noise to the aggregated gradient, providing rigorous privacy guarantees for model training.
\item \textbf{Frozen Embedding with Additive Noise (FN)}~\cite{Yu2021DifferentiallyPF}, a simple baseline that freezes the embedding layer and adds noise directly to its outputs. This strategy reduces the computational overhead of privacy reserving by fixing the large embedding layer, while the injected noise serves to perturb the input representations, offering an efficient privacy-preserving mechanism for the subsequent model layers.
\end{itemize}

In addition, we report results on two control setups to provide context for performance bounds:
\begin{itemize}[leftmargin=*, align=left]
\item \textbf{Multitask Learning (MTL)}, which trains on all tasks simultaneously and serves as an upper-bound reference.
\item \textbf{Sequential Finetuning (SeqFT)}, which learns tasks one after another sequentially without any mechanism to mitigate catastrophic forgetting.
\end{itemize}

\section{Implementation Details.} 
In our experiments, we train our models on A800-80GB GPUs. All methods employ the LLaMA-2-7B-hf \cite{touvron2023llama} for text-based tasks. For all baseline methods, we follow the implementation details and configurations from the original papers to ensure faithful reproduction. During training, we train each task for three epochs with a batch size of 32. For the hyperparameters of our approach, we set $\alpha = 0.5$, $\epsilon_{\text{upper}} = \lambda_{\text{max}} = 10$, $\epsilon_{\text{lower}} = \lambda_{\text{min}} = 1$, $\delta=1.0e-6$, $\lambda_{\text{unlearn}} = 1$, and $\theta = 0.6$. We use AdamW as the optimizer with a learning rate of $5.0e-4$ and employ a cosine learning rate scheduler.

\section{Theoretical Proof of Privacy Guarantee}
We present a formal analysis showing that our Token-level Dynamic Differential Privacy (TDP) mechanism, as defined in Section \ref{sect:Token-level Dynamic Differential Privacy}, satisfies $(\epsilon_{i},\delta)$-local differential privacy for each token $t_{i}$.

To rigorously establish this guarantee, we proceed in a structured manner. First, we recall the standard definition of $(\epsilon, \delta)$-local differential privacy, which serves as our formal privacy notion. Next, we bound the $\ell_2$ sensitivity of the clipped token embedding—a critical prerequisite for applying the Gaussian mechanism. With this sensitivity bound in hand, we invoke the well-known privacy guarantee of the Gaussian mechanism to derive the required noise scale $\sigma_i$ that ensures $(\epsilon_i, \delta)$-LDP for each token. Finally, we connect this analysis back to our dynamic sensitivity scoring framework by showing that the sensitivity score $\text{Score}(t_i)$ is properly bounded in $[0,1]$, which in turn guarantees that the derived privacy budget $\epsilon_i$ remains within a valid and interpretable range. Together, these components form a complete chain of reasoning that validates the per-token privacy guarantee of our TDP mechanism.

\begin{definition}[Local Differential Privacy]\hfill \\
A randomized mechanism $\mathcal{M}$ satisfies $(\epsilon, \delta)$-local differential privacy if for any two inputs $x$ and $x'$ differing in one entry (e.g., one token), and for any measurable set of outputs $\mathcal{S}$, it holds that:
\[
\Pr[\mathcal{M}(x) \in \mathcal{S}] \leq e^\epsilon \cdot \Pr[\mathcal{M}(x') \in \mathcal{S}] + \delta.
\]
\end{definition}

\begin{lemma}[Sensitivity of Clipped Embedding]\hfill \\
Let $e_i, e_i' \in \mathbb{R}^d$ be two embeddings corresponding to tokens $t_i$ and $t_i'$ respectively. If both are clipped such that $\|e_i\|_2 \le C$ and $\|e_i'\|_2 \le C$, then the $\ell_2$ sensitivity of the mechanism is bounded as:
\[
\Delta = \|\text{clip}(e_i, C) - \text{clip}(e_i', C)\|_2 \le 2C.
\]
\end{lemma}

\begin{proposition}[Per-token Differential Privacy Guarantee] \hfill \\
Let $\tilde{e}_i = \mathcal{M}(e_i)$ denote the perturbed embedding of token $t_i$. Then $\mathcal{M}$ satisfies $(\epsilon_i, \delta)$-differential privacy for each token $t_i$.
\end{proposition}

\begin{proof}
This proposition follows directly from the Gaussian Mechanism guarantee. The mechanism adds noise $\mathcal{N}(0, \sigma_i^2 I)$ calibrated to a sensitivity of $\Delta = 2C$ (from Lemma 1). To satisfy $(\epsilon_i, \delta)$-DP, the noise standard deviation must meet the condition:
\[
\sigma_i \ge \frac{2C \cdot \sqrt{2 \log(1.25/\delta)}}{\epsilon_i}.
\]
Our setting for $\sigma_i$, as defined in Equation 5 (Section 3.1), matches this bound precisely.
\renewcommand{\qedsymbol}{}
\end{proof}

\begin{lemma}[Range of Sensitivity Score] \hfill \\
The fused sensitivity score $\text{Score}(t_i)$ defined as:
\[
\begin{split}
\text{Score}(t_i) =\\ 1 - \exp \Bigl( - \bigl( & \alpha \cdot \text{Score}_1(t_i) + (1 - \alpha) \cdot \text{Score}_2(t_i) \bigr) \Bigr)
\end{split}
\]
which is bounded in $[0,1]$ for any $\alpha \in [0,1]$, provided that $\text{Score}_1(t_i) \ge 0$ and $\text{Score}_2(t_i) \ge 0$.
\end{lemma}

\begin{proof}
Since both $\text{Score}_1$ and $\text{Score}_2$ are non-negative, their convex combination is also non-negative. Hence,
\[
\text{Score}(t_i) = 1 - \exp(-z) \in [0, 1), \quad \text{where } z \ge 0.
\]
Thus, the score lies in $[0,1)$, which we clip to 1 for completeness.
\renewcommand{\qedsymbol}{}
\end{proof}

\begin{corollary}[Dynamic Privacy Budget Range] \hfill \\
Using $\text{Score}(t_i) \in [0,1]$, the dynamically assigned privacy budget
\[
\epsilon_i = \epsilon_{\text{lower}} + (\epsilon_{\text{upper}} - \epsilon_{\text{lower}})(1 - \text{Score}(t_i))^2
\]
is always in the range $[\epsilon_{\text{lower}}, \epsilon_{\text{upper}}]$.
\end{corollary}

\begin{remark}[Privacy Composition Across Sequence] \hfill \\
If needed, a user-level privacy budget for a token sequence $\{t_1, \dots, t_L\}$ can be computed via advanced composition. For example, the total privacy budget $\epsilon_{\text{total}}$ for $L$ tokens under sequential composition is:
\[
\epsilon_{\text{total}} \approx \sum_{i=1}^{L} \epsilon_i + \sqrt{2L \log(1/\delta')} \cdot \max_i \epsilon_i,
\]
with final failure probability $\delta + \delta'$. Alternatively, one may use RDP accounting for tighter bounds.
\end{remark}

\end{document}